\def\year{2019}\relax
\documentclass[letterpaper]{article} 
\usepackage{aaai19} 
\usepackage{times} 
\usepackage{helvet} 
\usepackage{courier} 
\usepackage{url} 
\usepackage{graphicx} 
\usepackage{amsmath}
\usepackage{amsthm}
\usepackage{microtype} 

\usepackage{amssymb}
\usepackage{subfig}
\usepackage{algorithm} 
\usepackage{algorithmic} 
\usepackage{graphicx}
\newtheorem{theorem}{Theorem} 
 
\newtheorem{corollary}{Corollary} 
\newtheorem{lemma}{Lemma} 
\newtheorem{proposition}{Proposition}

\newcommand\numberthis{\addtocounter{equation}{1}\tag{\theequation}}
\newcommand{\E}{\mathop{\mathbb{E}}}
\newcommand{\mne}{ \mathbf{MNE}}
\newcommand{\mnbe}{ \mathbf{MNBE}}
\begin{document}
	%
	\title{ Target Transfer Q-Learning and Its Convergence Analysis }
	
	\author{
		Yue Wang$^{\dagger}$\thanks{This work was done when the first author was visiting Microsoft Research Asia.} , Qi Meng$ ^{\ddagger} $, Wei Cheng$ ^{\ddagger} $, Yuting Liug$ ^{\dagger} $, Zhi-Ming Ma$ ^{\dagger} $$ ^{\S} $, Tie-Yan Liu$ ^{\ddagger} $ \\	
		$ ^{\dagger} $School of  Science, Beijing Jiaotong University, Beijing, China \{11271012, ytliu\}@bjtu.edu.cn\\
		$ ^{\ddagger} $Microsoft Research, Beijing, China \{meq, wche,Tie-Yan.Liu\}@microsoft.com\\
		$ ^{\S} $ Academy of Mathematics and Systems Science, Chinese Academy of Sciences, Beijing, China  mazm@amt.ac.cn \\			
	}

	\maketitle
	\begin{abstract}

		Reinforcement Learning (RL) technologies are powerful to learn how to interact with environments and have been successfully applied to variants of important applications. Q-learning is one of the most popular methods in RL, which uses temporal difference method to update the Q-function and can asymptotically learn the optimal Q-function. Transfer Learning aims to utilize the learned knowledge from source tasks to help new tasks. For supervised learning, it has been shown that transfer learning has the potential to significantly improve the sample complexity of the new tasks. Considering that data collection in RL is both more time and cost consuming  and Q-learning converges slowly comparing to supervised learning, different kinds of transfer RL algorithms are designed. However, most of them are heuristic with no theoretical guarantee of the convergence rate. Therefore, it is important for us to clearly understand when and how will transfer learning help RL method and provide the theoretical guarantee for the improvement of the sample complexity. In this paper, we propose to transfer the Q-function learned in the source task to the target in the Q-learning of the new task when certain safe conditions are satisfied. We call this new transfer Q-learning method \emph{target transfer Q-Learning}.		The safe conditions are necessary to avoid the harm to the new tasks brought by the transfer target and thus ensure the convergence of the algorithm.
		We study the convergence rate of the target transfer Q-learning. We prove that  if the two tasks are similar with respect to the MDPs, the optimal Q-functions of the two tasks are similar which means the error of the transferred target Q-function in the new task  is small.  Also, the convergence rate analysis shows that  the \emph{target transfer Q-Learning} will converge faster than Q-learning if the error of the transferred target Q-function is smaller than the current Q-function in the new task. 	Based on our theoretical results and the relationship between the Q error and the Bellman error, we design the safe condition as the Bellman error of the transferred target Q-function is less than   the current Q-function. Our experiments are consistent with our theoretical founding and verified the effectiveness of our proposed target transfer Q-learning method.

	\end{abstract}
	
	\section{Introduction}
	
	Reinforcement Learning (RL)  \cite{sutton1998reinforcement}  technologies are very powerful to learn how to interact with environments and have been successfully applied to variants of important applications, such as robotics, computer games and so on \cite{kober2013reinforcement,mnih2015human,silver2016mastering,bahdanau2016actor}.
	
	Q-learning \cite{watkins1989learning} is one of the most popular RL algorithms which uses  temporal difference method to update the Q-function. To be specific, Q-learning maps the current Q-function to a new Q-function by using Bellman operator and use the difference between these two Q-functions to update the Q-function. Since Bellman operator is a contractive mapping, Q-learning will converge to the optimal Q-function \cite{jaakkola1994convergence}. Comparing to supervised learning algorithms, Q-learning converges much slower due to the interactions with the environment. At the same time, the data collection is both very time and cost consuming in RL. Thus, it is crucial for us to utilize available information to save the sample complexity of Q-Learning.

	Transfer learning aims to improve the learning performance on a new task by utilizing knowledge/model learned from source tasks. Transfer learning has a long history in supervised learning \cite{li2009transfer,pan2010survey,oquab2014learning}. Recently, by leveraging the experiences from supervised transfer  learning, researchers developed different kinds of transfer learning methods for RL, which can be categorized into three classes: (1)\textit{ instance transfer} in which old data will be reused in the new task \cite{sunmola2006model,zhan2015online}; (2) \textit{representation transfer}  such as reward shaping and basis function extraction \cite{konidaris2006autonomous,Barreto2017SuccessorFF}; (3) \textit{parameter transfer} \cite{song2016measuring} in which the parameters of the source task will be partially merged into the model of the new task. While supervised learning is a pure optimization problem, reinforcement learning is a more complex control problem. To the best of our knowledge, most of the existing transfer reinforcement learning algorithms are heuristic with no theoretical guarantee of the convergence rate \cite{bone2008survey}, \cite{Taylor2009TransferLF} and \cite{lazaric2012transfer}. As mentioned by \cite{spector2018sample},  the transfer learning method potentially do not work or even harm to the new tasks  and we do not know the reason since the absence of the theory. Therefore, it is very important for us to clearly understand how and when transfer learning will help reinforcement learning save sample complexity.

	In this paper, we design a novel transfer learning method for Q-learning in RL with theoretical guarantee. Different from the existing transfer RL algorithms, we propose to transfer the Q-function learned in the source task as the   temporal difference update target  of the new task when certain safe conditions are satisfied. We call this new transfer Q-learning method \emph{target transfer Q-learning}. The intuitive motivation is  that when the two RL tasks are similar to each other, their optimal Q-function will be similar which means the   transferred target is better ( the error is smaller than the current Q-function ). Combine it with that a better target Q-function in Q-learning will help to accelerate the convergence, we may expect that  the \emph{target transfer Q-learning} method will outperform the Q-learning. The safe conditions are necessary to avoid the harm to the new tasks and thus ensure the convergence of the algorithm.
	
	We prove that target transfer Q-learning has the theoretical guarantee of convergence rate. Furthermore, if the two MDPs and thus the optimal Q-functions in the source and new RL tasks are similar, the target transfer Q-learning converges faster than Q-learning. 	
	To be specific, we prove the error of target transfer Q-learning consists of two errors: the initialization error  and the sampling error. Both of the errors are increasing with the  the  product of discount factor $\gamma$ and the \emph{relative Q-function error ratio }$ \beta $ (\emph{error ratio} for simplicity) which measures the relative error of the   target Q-function  comparing with the current Q-function in the new task. We called $ \gamma \beta  $ discounted relative Q-function error ratio(\emph{discounted error ratio} for simplicity).  The smaller the discounted error ratio is, the faster the convergence  is. And if the discounted error ratio is larger than 1, the convergence will no longer guaranteed. 
	
	If the two RL tasks are similar, the learned Q-function in the source task will be close to the optimal Q-function  comparing to the current Q-function in the new task. Thus, the discounted error ratio $ \gamma \beta  $ will be small(especially for the early stage) when we  transfer the learned Q-function from the source task to the target of the new task. Please note that the traditional Q-learning is a special case for target transfer Q-learning with constant discounted error ratio $\gamma$. 
	
	Therefore, our convergence analysis for target transfer Q-learning help us   design  the safe condition. We can transfer the target if it will lead the discounted error ratio $ \gamma \beta  $   smaller than $1$ . We call it \emph{error ratio} safe condition. Specifically, in the early stage of the training, the Q-function in the new task is not fully trained, the learned Q-function in the source task it a better choice with a smaller error ratio. With the updating of the Q-function in the new task, its error ratio becomes larger. When its discounted error ratio is close or larger than $1$, the safe condition will not be satisfied, and we will stop transferring the target to avoid the harm brought by the transfer learning. Following the standard way in Q-learning, we estimate the error ratio about the error of the Q-function w.r.t the optimal Q-function by the Bellman error.

	Our experiments on synthetic MDPs fully support our convergence analysis and verify the effectiveness of our proposed target transfer Q-Learning with error ratio safe condition.
 
	\section{Related Work}

	This section briefly outline related work in transfer learning in reinforcement learning.
	
	Transfer Learning in RL\cite{Taylor2009TransferLF} \cite{lazaric2012transfer} aims to improve learning in new MDP tasks by borrowing knowledge from a related but different learned MDP tasks.
	In paper \cite{Laroche2017TransferRL}, the authors propose to use instance transfer in the Transfer Reinforcement Learning with Shared Dynamics (TRLSD) setting in which only the reward function is different between MDPs.
	 In paper \cite{gupta2017learning}, the authors propose to use the representation transfer and learned the invariant feature space.
	 The papers \cite{Karimpanal2018SelfOrganizingMA,song2016measuring}  propose to use the parameter transfer  to guide the exploration or to initialize the Q-function of the new task directly.  
	 In paper \cite{al2017continuous}, the authors propose to use the meta-learning method to do transfer learning in RL. All these works are empirically evaluated and no theoretical analysis for the convergence rate.
		
	There are few works that have the convergence analysis.
	In paper \cite{Barreto2017SuccessorFF}, the authors use the representation transfer but only consider the    TRLSD  setting. \cite{zhan2015online} propose a method by using instance transfer. They   gives the theoretical analysis of the asymptotic  convergence and no finite sample performance guarantee. 

	\section{Q Learning Background}
	
	Consider the reinforcement learning problem with Markov decision process (MDP) $M \triangleq (\mathcal{S},\mathcal{A},P,r,\gamma) $, where $ \mathcal{S} $ is the state space, $ \mathcal{A} $ is the action space, $ P= \{P_{s,s'}^a; s,s’\in \mathcal{S}, a\in\mathcal{A}\} $ is the transition matrix and $ P_{s,s'}^a $ is the transition probability from state $ s $ to state $ s' $ after taking action $ a $, $ r=\{r(s,a );s\in \mathcal{S},a\in\mathcal{A}\}$ is the reward function and $r(s,a)$ is the reward received at state $s$ if taking action $a$, and $ 0<\gamma<1 $ is the discount factor. 
	A policy $ \pi: \mathcal{A}\times\mathcal{S}\to [0,1]$ indicates the probability to take each action at each state. Value function for policy $ \pi $ is defined as: $ V^{\pi}(s)\triangleq E\left[ \sum_{t=0}^{\infty}\gamma^t r(s_t,a_t)|s_0 = s,\pi \right] $. Action value function for policy $ \pi $ is also called Q-function and is defined as: $$ Q^{\pi}(s,a)\triangleq E\left[ \sum_{t=0}^{\infty}\gamma^t r(s_t,a_t)|s_0 = s,a_0 =a,\pi \right]. $$ Without loss of generality, we assume that the rewards all lie between 0 and 1. The optimal policy is denoted $ \pi^* $ and has value function $ V^*_M(s) $ and Q value function $ Q^*_M(s,a) $.
	
	As we know, the Q-function in RL satisfies the following Bellman equation:$$ Q^\pi(s,a) = r(s ,a )+\gamma\E_{\substack{\tilde{a}\sim \pi(a|s)\\s'\sim P(s'|s,a)}}\left[ Q^\pi(s', \tilde{a}) | s_t=s \right] $$
	Denote the right hand side(RHS) of the equation as $ T^\pi Q^\pi(s,a) $ , $ T^\pi $ is called Bellman operator for policy $\pi$. Similar, consider the optimal Bellman equation:$$ Q^*(s,a) = r(s,a) + \gamma\E_{\substack{\tilde{a}\sim \pi(a|s)\\s'\sim P(s'|s,a)}}\left[ Q^*(s', \tilde{a}) | s_t=s \right] $$
	(RHS) of the equation is been denoted as $ T^\pi Q^\pi(s,a) $,$ T^* $ is called optimal Bellman operator.
	It can be proved that the optimal Bellman operator is a contraction mapping for the Q-function. We know that there is an unique fix point which is optimal Q-function by contraction mapping theorem. 
	Q-learning algorithm is designed by the above theory. Watkins introduced the Q-learning algorithm to estimate the value of state-action pairs in discounted MDPs \cite{watkins1989learning} :
	\begin{align*}
	&Q_{t+1}(s,a)\\
	\small& = (1-\alpha_t)Q_t(s,a) + \alpha_t \left(r_t(s,a) + \gamma \max_{\tilde{a}}Q_t(s' ,\tilde{a}) \right)
	\end{align*}
	We introduce the max norm error to measure  the quality of   Q-function:
  $$ \mne(Q ) = \max_{s,a}\vert Q(s,a) - Q^*(s,a) \vert . $$

	\section{ Target  Transfer Q-Learning}
	
    First of all, we formalize transfer learning in RL problem. Secondly, We propose our new transfer Q-learning method  Target Transfer Q-Learning (TTQL) and introduce the intuition.
	
	Transfer Learning in RL\cite{Taylor2009TransferLF} \cite{lazaric2012transfer} aims to improve learning in new MDP tasks by borrowing knowledge from a related but different learned MDP tasks.  
	
	According to the definition of MDPs, $M \triangleq (\mathcal{S},\mathcal{A},P,r,\gamma) $, we consider the situation that two MDPs are different in transition probability $ P $, reward function $ r $ and discount factor $ \gamma $. Assume there are two MDPs: source MDP $ M_1 = (\mathcal{S},\mathcal{A},P_1,r_1,\gamma_1) $ and new MDP $ M_2 = (\mathcal{S},\mathcal{A},P_2,r_2,\gamma_2) $, $ Q^*_1 $ and $ Q^*_2 $ are the corresponding optimal Q-functions. Let $ M_1 $ be the source domain and we have already learned the $ Q^*_1 $. The goal of transfer in RL considered in this work is how  we can use the information of $ M_1 $ and $ Q^*_1 $ to achieve learning speed improvement in $ M_2 $.
	
	To solve the problem mentioned above, we propose to use TTQL method. TTQL use the Q-function learned from the source task as the target Q-function in the new task when safe conditions satisfied. The safe condition ensures that the transferred target only appears if it can help to accelerate the training. Otherwise we will replace it with the current Q-function in the new MDP's learning progress.  We describe the TTQL in Algorithm 1.
		\begin{algorithm}[h] \label{alg} 
		\caption{ Target  Transfer Q Learning }
		\begin{algorithmic}[1] 
			\REQUIRE initial Q-learning $ Q_1 $ , source task learned Q-learning $ Q^*_{source} $, total step $ n $
			\FOR{$ t = 1, \dots, n $}
			\STATE $ \alpha_t = \frac{1}{t+1} $
			\STATE flag = \texttt{safe-condition}($Q^*_{source} , Q_{t}(\cdot,\cdot) $) 
			\IF{flag = True} 
			\STATE $ Q_{target} = Q^*_{source} $
			\ELSE
			\STATE $ Q_{target} = Q_{t} $
			\ENDIF
			\FOR{$ s \in \mathcal{S} , a \in \mathcal{A} $}
			\STATE $ Q_{t+1}(s,a) = ( 1 - \frac{1}{n})Q_{t}(s,a) + \frac{1}{n} \left(r(s,a) + \gamma \max_{\tilde{a}} Q_{target} (s',\tilde{a})\right) $
			\ENDFOR
			\ENDFOR
			\ENSURE 
			$ Q_{n+1} $ 
		\end{algorithmic}
	\end{algorithm}

 	 The intuitive motivation is  that when the two RL tasks are similar to each other, their optimal Q-function will be similar. Thus the   transferred target is better ( the error is smaller than the current Q-function ) and the better target can help to accelerate the convergence.

	We define the distance between two MDPs as $ \Delta(M_1 , M_2) $ $$ \Delta(M_1 , M_2) = \max_{s,a}|Q_1^*(s,a) - Q_2^*(s,a)|.$$
		The following Proposition \ref{propq*diff} shows the relation between the distance of two MDPs and the component of two MDPs.
	\begin{proposition}\label{propq*diff}
		Assume two MDPs, $ M_1 = (\mathcal{S},\mathcal{A},P_1,r_1,\gamma_1) $ and $ M_2 = (\mathcal{S},\mathcal{A},P_2,r_2,\gamma_2) $, Let the corresponding optimal Q-functions be $ Q^*_1 $ and $ Q^*_2 $, then we have
		\small\begin{align*}
		&\Delta(M_1 , M_2) = \Vert Q_1^* -Q_2^* \Vert_\infty \le \tilde{\Delta}(M_1 , M_2) \numberthis\\
		& \triangleq \frac{\Vert r_1 - r_2 \Vert_\infty}{1 - \gamma'} + \frac{\gamma''\Vert r' \Vert_\infty}{(1 - \gamma'')^2} \Vert P_1 - P_2 \Vert_\infty +\frac{\vert \gamma_1 - \gamma_2 \vert }{(1 - \gamma_1)(1-\gamma_2)}\Vert r'' \Vert_\infty.\\ 
		\end{align*}\normalsize 
		for $\forall (\gamma', \gamma'', r', r'') \in \Omega $, where $ \Omega $ is the available combination of the $ (\gamma_1 , \gamma_2 , \gamma_1 , \gamma_2) $. 
	\end{proposition}
	
	
		\begin{proof}
		Without loss of generality, we assume $ \gamma_1 \le \gamma_2 $, $ \Vert r_2 \Vert_\infty \le \Vert r_1 \Vert_\infty $, we will show that other cases can be proved similarly. We define the following auxiliary MDPs: $ \hat{M_3} = (\mathcal{S},\mathcal{A},P_1,r_2,\gamma_1) $, $ \hat{M_4} = (\mathcal{S},\mathcal{A},P_2,r_2,\gamma_1) $, and let the corresponding optimal Q-functions be $ Q^*_3 $ and $ Q^*_4 $. We have
		\begin{align}
		&\Vert Q_1^* -Q_2^* \Vert_\infty \numberthis\\
		& = \Vert Q_1^* - Q^*_3 + Q^*_3 - Q^*_4 + Q^*_4 - Q_2^* \Vert_\infty \\
		&\le \Vert Q_1^* - Q^*_3 \Vert_\infty + \Vert Q^*_3 - Q^*_4\Vert_\infty + \Vert Q^*_4 - Q_2^* \Vert_\infty 
		\end{align}
		Notice that in each term, two MDPs are only different in one component. Using the results of \cite{csaji2008value}, we have that
		$ \Vert Q_1^* - Q^*_3 \Vert_\infty \le \frac{\Vert r_1 - r_2 \Vert_\infty}{1 - \gamma_1} $, $ \Vert Q^*_3 - Q^*_4\Vert_\infty \le \frac{\gamma_1\Vert r_2 \Vert_\infty}{(1 - \gamma_1)^2} \Vert P_1 - P_2 \Vert_\infty $, $\Vert Q^*_4 - Q_2^* \Vert_\infty \le \frac{\vert \gamma_1 - \gamma_2 \vert }{(1 - \gamma_1)(1-\gamma_2)}\Vert r_2 \Vert_\infty $. Combine the above upper bound  and set $ \gamma'=\gamma_1, \gamma''=\gamma_1, r'=r_2, r''=r_2  $, we can get the in-equation  (1).

		In other situation, we can construct auxiliary MDPs like above and use the similar procedure to prove the theorem. 
		After traversing all the available combination of the $ (\gamma_1 , \gamma_2 , \gamma_1 , \gamma_2) $,    we can prove the Proposition \ref{propq*diff} 

	\end{proof}

	By the Proposition \ref{propq*diff}, we can conclude that 	if the two RL tasks are similar, in the sense of that the component of two MDPs are similar, the learned Q-function in the source task will be close to the optimal Q in the new task. 

	  A question is that when to transfer the target will have performance guarantee. Here, we need safe conditions which are   necessary to avoid the harm to the new tasks and thus ensure the convergence of the algorithm. We can now heuristically  relate it  to the distance between two MDPs and the current learning quality. The concrete value of the safe condition need to further investigate through quantified theoretical analysis and we present these result in the following section.

	\section{Convergence Rate of TTQL}
	In this section, we present the convergence rate of the Target Transfer Q Learning (TTQL) and make discussions for the key factor that influence the convergence. Theorem \ref{q'} analysis the convergence of the target transfer Q learning. Theorem \ref{thmsumw2} and \ref{alpha} analysis two key factors of the convergence rate. Theorem \ref{thmkey} discuss the convergence rate for the TTQL totally.
	
	First of all,  Theorem \ref{q'} analysis the convergence rate for the  target transfer method which is 
	\small$$ Q_{t+1}(s,a) = ( 1 - \frac{1}{n})Q_{t}(s,a) + \frac{1}{n} \left(r(s,a) + \gamma \max_{\tilde{a}} Q_{target} (s',\tilde{a})\right) $$\normalsize

	For simplicity, we denote $ E_n = \mne(Q_n) $. We denote the   error ratio  $ \beta_n = \frac{\mne(Q_{target})}{E_n}   $ and $  \beta  $ if we do not specify the learning steps $ n $. 
	\begin{theorem}\label{q'}
	
		we denote $w_k( \beta _{n-k:n} ) = \frac{\prod_{i=n-k}^{n-1} (i+\gamma\beta_i)}{\prod_{i=n-k}^{n } i}$, $ \alpha_n = \frac{\prod_{i=1}^{n-1} (i+\gamma\beta_i )}{\prod_{i=2}^{n } i} $. If $0\le\beta_n \le 1$, then with probability $ 1-\delta $ we have 
		\small\begin{align*}
		&E_n 
		\le&\underbrace{\alpha_n E_1}_{\text{initialization error}} + \underbrace{ \sqrt{\frac{\ln1/\delta\sum_{k=0}^{n-1}w_k^2(\beta_{n-k:n}) }{2 }}}_{\text{sampling error}}.
		\end{align*}\normalsize
	\end{theorem}
	
	Before showing the proof of Theorem \ref{q'}, we first introduce a  modified Hoeffding inequality lemma which bounds the distance between the weighted sum of the bounded random variable and its expectation.
			\begin{lemma}\label{lemhoeff}
		Let $ a< x_i <b$ almost surely , $S_n = \sum_{i=1}^{n}w_ix_i $, then we have
		\begin{align}
		S_n - E[S_n] \le \sqrt{\frac{1}{2} \log\frac{1}{\delta}\sum_{k=1}^{n}w_k^2(b-a)^2}. \label{invershoeff} 
		\end{align}
	\end{lemma}
	\begin{proof}
		We first prove the inequality 			$ 	\mathbb{P}\left( S_n - E[S_n] \ge \epsilon \right) \le exp\left( -\frac{2\epsilon^2}{\sum_{k=1}^{n}w_k^2(b-a)^2} \right) \label{hoeff} $
		
		For $ s,\epsilon \ge 0 $, Markov's inequality and the independence of $ x_i $ implies
		\small\begin{align} 
		& \mathbb{P}\left(S_{n}-\mathrm{E}\left [S_n \right ]\geq \epsilon \right)\\
		&= \mathbb{P} \left (e^{s(S_n-\mathrm{E}\left [S_n \right ])} \geq e^{s\epsilon} \right)\\
		&\leq e^{-s\epsilon} \mathrm{E} \left [e^{s(S_{n}-\mathrm{E}\left [S_n \right ])} \right ]\\
		&=e^{-s\epsilon} \mathrm{E} \left [e^{s( \sum_{i=1}^{n}w_ix_i-\mathrm{E}\left [ \sum_{i=1}^{n}w_ix_i \right ])} \right ]\\
		&= e^{-s\epsilon} \prod_{i=1}^{n}\mathrm{E} \left [e^{sw_i(x_i-\mathrm{E}\left [x_{i}\right])} \right ]\\
		&\leq e^{-s\epsilon} \prod_{i=1}^{n} e^{\frac{s^2 w_i^2(b -a )^2}{8} } \\
		&= \exp\left(-s\epsilon+\tfrac{1}{8} s^2 (b -a)^{2}\sum_{i=1}^{n}w_i^2\right).
		\end{align} 
		Now we consider the minimum of the right hand side of the last inequality as a function of $ s $, and denote
		$$ g(s)=-s\epsilon+\tfrac{1}{8} s^2 (b -a)^{2}\sum_{i=1}^{n}w_i^2\ $$
		Note that g is a quadratic function and achieves its minimum at $ s=\frac{4\epsilon}{(b-a)^2\sum_{i=1}^{n}w_i^2} $,
		Thus we get 
		\begin{align}
		\mathbb{P}\left(S_{n}-\mathrm{E}\left [S_n \right ]\geq \epsilon \right) \le exp\left( -\frac{2\epsilon^2}{\sum_{k=1}^{n}w_k^2(b-a)^2} \right)
		\end{align}
		We can easily obtain the second part of the Lemma \ref{lemhoeff} by inverse the inequality.  
	\end{proof}

	\begin{proof}[Proof of Theorem \ref{q'}]
		Our analysis are derived based on the following synchronous generalized Q-learning setting. Compare with the traditional synchronous Q-learning \footnote{It is the same as the commonly used setting or more general(\cite{pmlr-v70-asadi17a}, \cite{even2003learning}, \cite{azar2013speedy} \cite{haarnoja2017reinforcement}).}, we replace the target Q-function as the independent Q-function $ Q'(s,a) $ rather than the current one $ Q_n(s,a) $.
		\begin{align*}
		&\forall s,a~:~ Q_0(s,a)=q(s,a)\\
		&\forall s,a~:~ Q_n(s,a) = \\
		&~~~~~~ ( \frac{n-1}{n})Q_{n-1}(s,a) + \frac{1}{n} \left(r(s,a) + \gamma \max_{\tilde{a}} Q'_{n-1} (s',\tilde{a})\right) \numberthis
		\end{align*}
		Let $ Q_n'(s,a) $ satisfied the following condition  ,
		\begin{align}
		0 \le \frac{	\max_{s,a}\left(Q_n'(s,a) - Q^*(s,a) \right)}{\max_{s,a}\left(Q_{n}(s,a) - Q^*(s,a) \right)} \le 1 \label{proofbeta} 
		\end{align}
			Note that if we set $ Q'_n(s,a) = Q^*_{source} $, we can verify $0 \le  \beta_n \le 1 $  according to inequality \ref{proofbeta}.
		First of all, we decompose the update role,
		\begin{align*}
		& Q_n(s,a) \\
		& =\frac{n-1}{n}Q_{n-1}(s,a) + \frac{1}{n}\left[r(s,a) + \gamma\max_{\tilde{a}}Q'_{n-1}(s',\tilde{a})\right] \\
		&= \frac{n-1}{n}Q_{n-1}(s,a) + \frac{1}{n}\left[r(s,a) + \gamma\max_{\tilde{a}}Q^*(s',\tilde{a}) \right. \\
		& ~~~~~~~~~~~~~~~ \left.+\gamma\max_{\tilde{a}}Q'_{n-1}(s',\tilde{a}) - \gamma\max_{\tilde{a}}Q^*(s',\tilde{a}) \right] 
		\end{align*}
		If we denote $\epsilon_n(s,a) = Q_n(s,a) - Q^*(s,a) $, $ x(s') = \gamma\max_{\tilde{a}}Q^*(s',\tilde{a}) $ and recall the definition of $ \beta_n $ we can have   
		\begin{align*}
		& \epsilon_n(s,a) \\
		\le & \frac{n-1}{n}\epsilon_{n-1}(s,a)+\frac{1}{n}\left[x(s') - \mathbb{E}_{s'}x(s')\right] + \frac{1}{n}\gamma\beta_n\epsilon_{n-1}(s',\tilde{a}) \\
		\le & \frac{n-1}{n}\epsilon_{n-1}(s,a) +\frac{1}{n}\left[x(s') - \mathbb{E}_{s'}x(s')\right] + \frac{1}{n}\gamma\beta_nE_{n-1} 
		\end{align*}
		The last step is right because $ \epsilon_n(s,a) \le E_n$ for $ \forall s,a $. Taking maximization of the both sides(RHS) of the inequality and using recursion of $ E $ we can have
		 \begin{align*} 
		& E_n \le \frac{n-1+\gamma\beta_n}{n}E_{n-1} + \frac{1}{n}\left[x(s') - \mathbb{E}_{s'}x(s')\right] \\
		 &\le \frac{\prod\limits_{i=1}^{n-1} (i+\gamma\beta_i)}{\prod\limits_{i=2}^{n } i}E_1 + \sum_{k=1}^{n-1}\frac{\prod\limits_{i=n-k}^{n-1} (i+\gamma\beta_i)}{\prod\limits_{i=n-k}^{n } i}[x(s'_k) - \E_{\substack{s'}}x(s')]\\
		&= \alpha_n E_1 + \sum_{k=1}^{n-1}w_k(\beta)[x(s'_k) - \mathbb{E}_{s'}x(s')]\\
		\end{align*} 
		According to Lemma \ref{lemhoeff}(weighted Hoeffding inequality), with probability 1-$ \delta $, we have
		\begin{align}
		E_n\le&\alpha_nE_1 + \sqrt{\frac{\ln1/\delta\sum_{k=0}^{n-1}w_k^2(\beta_{n-k:n}) }{2 }}
		\end{align} 
	\end{proof}

	The convergence result reveals the how the error ratio  $ \beta $   influence the convergence rate. In short, if we can find a better target Q-function, we can learn much more faster.

	We can see from the Theorem \ref{q'} that there are two key factors that influence the convergence rate. One is the initialization error   $\alpha_n E_1 $, the other one is the sampling error $ \sqrt{\frac{\ln1/\delta\sum_{k=0}^{n-1}w_k^2(\beta_{n-k,n}) }{2 }}$. To make it clear, we analysis the order of these two terms in \ref{thmsumw2} and \ref{alpha} respectively.

	\begin{theorem}\label{thmsumw2}
		 Denote $w_k(\beta_{n-k:n}) = \frac{\prod_{i=n-k}^{n-1} (i+\gamma\beta_i )}{\prod_{i=n-k}^{n } i}$, and $ \beta_i \le \beta^* \text{ for } \forall i \le n $,  we have
		\small\begin{align*}
		\sum_{k=0}^{n-1}\left(w_k(\beta_{n-k,n})\right)^2 \le 
		\left\{
		\begin{array}{lr}
		\frac{e^{ 2\gamma\beta^* }}{n^{2-2\gamma\beta^*}}\left(\frac{n^{1-2\gamma\beta^*}}{1-2\gamma\beta^*} - \frac{1}{1-2\gamma\beta^*}+1\right), \\
		~~~~~~~~~~~~~~~~~~~~~~~~~~~~~~~~~~~~~~~~~ \gamma\beta^*\not=0.5 \\ 
		\frac{(n-2)^{2\gamma\beta^*}}{n^2} e^{2\gamma\beta^*}(1 + \ln(n)),\\
		~~~~~~~~~~~~~~~~~~~~~~~~~~~~~~~~~~~~~~~ \gamma\beta^*=0.5 
		\end{array}
		\right..
		\end{align*}\normalsize
			\end{theorem}
		Based on the results of Theorem 2, we can get the following corollary directly.
		\begin{corollary}
			The  order of $ \sum_{k=0}^{n-1}\left(w_k(\beta_{n-k:n})\right)^2 $  is:\\
				  $ \mathcal{O}(\frac{1}{n}) $, if  $ \gamma\beta^* < 0.5 $,.\\
			  $ \mathcal{O}(\frac{1}{n^{2-2\gamma\beta^*}}) $, if  $ 0.5<\gamma\beta^* <1 $.\\ 
		  $ \mathcal{O}(\frac{1}{n^{2-2\gamma\beta^*}}\ln(n)) $,  	if $ \gamma\beta^* = 0.5 $.\\ 
			The sufficient condition for the  $ \lim_{n\to \infty}\sum_{k=0}^{n-1}\left(w_k(\beta^*)\right)^2 = 0 $ is  $ \gamma\beta^* <1 $    \\
		\end{corollary}

Before showing the proof of Theorem 2, we first introduce a Lemma which will be used.
	\begin{lemma}\label{sumint}
	If $ a < b $, $ \sum_{i=a}^{b}\frac{1}{i}\le \frac{1}{a} + \ln(b) - \ln(a) $.	
\end{lemma}

\small\begin{proof}
	\begin{align*}
	\sum_{i=a}^{b}\frac{1}{i}&\le \frac{1}{a} + \sum_{i=a+1}^{b}\frac{1}{i} \le \frac{1}{a} + \sum_{i=a+1}^{b}\int_{k=i-1}^{i}\frac{1}{k}dk  \\
	&\le  \frac{1}{a} +  \int_{k=a}^{b}\frac{1}{k}dk   \le \frac{1}{a} + \ln(b) - \ln(a)
	\end{align*}
\end{proof}\normalsize

	\begin{proof}[Proof of Theorem \ref{thmsumw2}]
		\small \begin{align*}
		&\sum_{k=0}^{n-1}\left(w_k(\beta_{n-k:n})\right)^2 \le \sum_{k=0}^{n-1}\left( \frac{\prod_{i=n-k}^{n-1} (i+\gamma\beta^*)}{\prod_{i=n-k}^{n } i}\right)^2 \\
		&\underbrace{=}_{(a)}\sum_{k=0}^{n-1}\exp \left\lbrace 2\left[ \sum_{i=n-k}^{n-1}\ln(i+\gamma\beta^*) -\sum_{i=n-k}^{n}\ln i \right] \right\rbrace \\
		&\underbrace{=}_{({b})}\frac{1}{n^2}\sum_{k=0}^{n-1}\exp \left\lbrace 2 \sum_{i=n-k}^{n-1}\left[ \ln(i+\gamma\beta^*) - \ln i \right] \right\rbrace \\
		&\underbrace{\le}_{(c)} \frac{1}{n^2}\sum_{k=0}^{n-1}\exp \left\lbrace 2 \sum_{i=n-k}^{n-1}\frac{ \gamma\beta^*}{i} \right\rbrace \\
		&\underbrace{\le}_{(d)} \frac{1}{n^2}\sum_{k=0}^{n-1}\exp \left\lbrace 2\gamma\beta^* \left[\ln(n-2) - \ln(n-k) + 1 \right] \right\rbrace \\
		&= \frac{(n-2)^{2\gamma\beta^*}}{n^2} e^{2\gamma\beta^*}\sum_{k=0}^{n-1}\frac{1}{(n-k)^{2\gamma\beta^*}} \\
		& = \frac{(n-2)^{2\gamma\beta^*}}{n^2} e^{2\gamma\beta^*}\sum_{t=1}^{n }\frac{1}{ t ^{2\gamma\beta^*}} \numberthis \label{sumw2}
		\end{align*}\normalsize
		We rewrite the product  term in (a) into the   summarization  term. Then we drop one term outside of the summarization to align the $ i $ sum from $ n-k $ to $ n-1 $ in (b). (c) follows the concave property of the $ \ln $ function. (d) follows the relation between summarization and integral as shown in Lemma \ref{sumint}. The last two terms is right because we only rearrange the term and write it simply.
		\indent
		
		If $ \gamma\beta^* = 0.5 $, $ 2\gamma\beta^* =1 $, 
		\begin{align*}
		\sum_{k=0}^{n-1}\left(w_k(\beta_{n-k:n})\right)^2 \le \frac{1}{n^{2-2\gamma\beta^*} }e^{\frac{2\gamma\beta^*}{n-1}}(1 + \ln(n))
		\end{align*}
		If $ \gamma\beta^* \not = 0.5 $,
		\small\begin{align*}
		\small&\sum_{k=0}^{n-1}\left(w_k(\beta^*)\right)^2 \le \underbrace{\frac{1}{n^{2-2\gamma\beta^*}}}_{(e)} \underbrace{e^{2\gamma\beta^*}}_{(f)}\left(\underbrace{\frac{n^{1-2\gamma\beta^*}}{1-2\gamma\beta^*}}_{(g)} - \underbrace{\frac{1}{1-2\gamma\beta^*}+1}_{(h)}\right) \label{key} \normalsize
		\end{align*}\normalsize
		\noindent
		Note that term (f) is a constant.\\
		If $ \gamma\beta^* < 0.5 $, term(g) will dominant the order, $ \sum_{k=0}^{n-1}\left(w_k(\beta_{n-k:n})\right)^2 $ will be $ \mathcal{O}(\frac{1}{n}) $.\\
		If $ \gamma\beta^* > 0.5 $, term(h) will dominant the order, $ \sum_{k=0}^{n-1}\left(w_k(\beta_{n-k:n})\right)^2 $ will be $ \mathcal{O}(\frac{1}{n^{2-2\gamma\beta^*}}) $.\\ 
		If $ \gamma\beta^* = 0.5 $, $ \sum_{k=0}^{n-1}\left(w_k(\beta_{n-k:n})\right)^2 $ will be $ \mathcal{O}(\frac{1}{n^{2-2\gamma\beta^*}}\ln(n)) $.\\ 
		In all case, the (\ref{key}) will converge to 0 as $ n $ will go to $ \infty $.
	\end{proof}

			Note that if $ \gamma\beta^* <1 $. The theorem \ref{thmsumw2} shows, $ \sum_{k=0}^{n-1}w_k^2 $ converges to 0 and the convergence rate is highly related to the $ \gamma\beta^* $. 		
	The next theorem shows the upper bound of the coefficient $ \alpha_n $ in initialization error.
	\begin{theorem}\label{alpha}
		Denote $ \alpha_n = \frac{\prod_{i=1}^{n-1} (i+\gamma\beta_i )}{\prod_{i=2}^{n } i} $, and $ \beta_i \le \beta^* \text{ for } \forall i \le n $,
		we can bound $ \alpha_n $ as:
		\begin{align}
		\alpha_n \le \frac{(n-1)^{\gamma\beta^*}}{n}(1+\gamma\beta^*)e^{(0.5-\ln2)\gamma\beta^*} = \frac{C^1_{\gamma,\beta^*}}{n^{1-\gamma\beta^*}}.
		\end{align}
		where $ C^1_{\gamma,\beta^*} =(1+\gamma\beta^*)e^{(0.5-\ln2)\gamma\beta^*} $   is a constant。
	\end{theorem}

	\begin{proof}[Proof of Theorem \ref{alpha}]
		 \small\begin{align}
		&\alpha_n \le \frac{\prod_{i=1}^{n-1} (i+\gamma\beta^* )}{\prod_{i=2}^{n } i} \\
		&= \exp\left\{ \sum_{i=1}^{n-1}\ln(i+\gamma\beta^*) - \sum_{i=2}^{n}\ln i \right\}\\
		&=(1+\gamma\beta^*)\exp\left\{ \sum_{i=2}^{n-1} \left( \ln(i+\gamma\beta^*) - \ln i \right) - \ln n\right\}\\
		& \le (1+\gamma\beta^*)\exp\left\{ \sum_{i=2}^{n-1} \left( \frac{\gamma\beta^*}{i} \right) - \ln n\right\}\\ 
		&\le (1+\gamma\beta^*)\exp\left\{ \gamma\beta^*(0.5 + \ln(n-1) - \ln2 )- \ln n\right\} \\
		& \le \frac{(n-1)^{\gamma\beta^*}}{n}(1+\gamma\beta^*)e^{(0.5-\ln2)\gamma\beta^*}
		\end{align}\normalsize
		We rewrite the product  term in the second equation into the   summarization term. The third equation is rearrange the terms. The first inequality follows the concave property of $ \ln $ function. The second inequality follows the relation between summarization and integral(Lemma \ref{sumint}).
	\end{proof}
		Note that if $ \gamma\beta^* <1 $. The theorem \ref{alpha} shows, $ \alpha_n $ converge to 0 and the convergence rate is in order $ \mathcal{O}(\frac{1}{n^{1-\gamma\beta^*}}) $.

	Combining Theorem \ref{q'}, \ref{thmsumw2} and \ref{alpha}, we have the following Theorem:\\
	\begin{theorem}\label{thmkey}
		The TTQL will converge if we set the \textbf{safe condition} as $$      \hat{\beta}_n = \frac{\Delta(M_1, M_2)}{E_n}  \le 1.  $$ And the convergence rate is: 
		\begin{align}
		E_n \le
					\left\{
			\begin{array}{lr}
			 \mathcal{O}(\frac{1}{n^{1-\gamma\beta}}E_1 + \sqrt{\frac{1}{n}} ) ,  &if ~  \gamma\beta < 0.5 \\
			\mathcal{O}(\frac{1}{n^{1-\gamma\beta}}E_1 + \frac{1}{n^{1-\gamma\beta}}\sqrt{\ln n}) , & if~   \gamma\beta = 0.5 \\
			\mathcal{O}(\frac{1}{n^{1-\gamma\beta}}E_1 + \frac{1}{n^{1-\gamma\beta}}) ,  &if ~ 0.5 <  \gamma\beta < 1\\
			\end{array}
			\right..
		\end{align}

	\end{theorem} 

Note that if the safe condition is satisfied, we set $ Q_{target} = Q_{source}^* $ and $ \beta $

	We would like to make the following discussion:  
	
	\textbf{(1) The distance between two MDPs influence the convergence rate.} According to the Proposition \ref{propq*diff}, if two MPDs have the similar components($ P $, $ r $, $\gamma$), the optimal Q-function of these two MDPs will be closed. The discounted error ratio $ \gamma\beta_n $ will be  relatively small in this situation and the convergence rate will be improved.
	
		\textbf{(2) Q-learning is the special case.}	Please note that the traditional Q-learning is a special case for target transfer Q-learning with $ Q_{target} = Q_{n-1} $. Thus the error ratio is a constant and   $\beta_n = 1$ and our results reduce to the previous \cite{szepesvari1998asymptotic}. It shows that  if the $ \beta < 1 $ in TTQL,the TTQL converge faster than traditional Q-learning.

		 \textbf{(3) The TTQL method do converge with the safe condition.}   As shown in  Theorem \ref{thmkey}, the TTQL method will converge.    And the convergence rate changes under  different discounted error ratio $  \gamma\beta $.  The smaller $ \gamma\beta $ will lead to a quicker convergence rate.	Intuitively, smaller $ \beta $ means  that $ Q' $ provides more information about the optimal Q-function. Besides, the discount factor $ \gamma $ can be viewed as the "horizon" of the infinite MDPs. Smaller $ \gamma $ means that the expected long-term return is less influenced by the future information and the immediate reward is assigned more weights.

	\textbf{(4) Safe condition is necessary.} As mentioned above,  the safe condition is defined as   $ \hat{\beta}_n \le 1  $. 			 
	If the safe condition   is satisfied, we set $ Q_{target} = Q_{source}^* $ and $ \gamma\beta_n = \gamma\hat{\beta}_n \le \gamma   < 1 $. 
	If    safe condition   is not satisfied, we set $ Q_{target} = Q_{n} $ and  $ \gamma\beta_n = \gamma < 1 $. So with the safe condition, TTQL algorithms do converge at any situation.
	 At the beginning of the new task training, due to the large error of the current Q-function, $ \beta_n = \hat{\beta}_n $ will be relatively small and the transfer learning will be  greatly helpful. Speedup would come down as the error of current Q-function, become smaller. Finally when $ \beta $ is equal to or larger than one we need to remove the transfer Q target which means to set $ \beta=1 $ to avoid the harm brought by the transfer learning.

	\section{Discussion for Error Ratio Safe Condition}

	Until now, we can conclude that TTQL will   converge.  TTQL method need the safe condition to guarantee the convergence. In this section, We discuss   the safe conditions.
	
	At the beginning, we propose the safe condition is that can guarantee the algorithms convergence generally. Heuristically, the safe condition is related to the distance between two MDPs and the quality of the current value function. Then according to the Theorem \ref{q'}, we know  that the safe condition is $  \hat{\beta}_n \le 1 $ which we called  error ratio safe condition. 
Under the transfer learning in RL setting, it means that  the distance between two MDPs need to be smaller than  the  error of the current Q-function.
 In the real algorithms, it is impossible to calculate the error of the current Q-function $ \mne(Q_n) $ and the distance between two MDPs precisely. However it is easy to calculate the bellman error $ \mnbe(Q(s,a)) = \max_{s,a} \left\vert Q(s,a) - (r(s,a) + \gamma E_{s'}\max_{\tilde{a}}(Q(s' ,\tilde{a})))\right\vert$.
We can prove that these two metrics follow the relationship  as:	
	$$\mne(Q) \le \frac{\mnbe(Q)}{1-\gamma}.	$$
 Following the standard way in Q-learning, we estimate the error ratio about the error of the Q-function w.r.t the optimal Q-function by the Bellman error.

		\begin{algorithm}[H] \label{sc2} 
		\caption{Error Ratio Safe Condition  }
		\begin{algorithmic}[1] 
			\REQUIRE leared $Q_1^*$ , current Q-function $Q_n $ 
			\IF {$ \mnbe(Q_1^*) \le \mnbe(Q_n)$}
			\STATE flag = True
			\ELSE
			\STATE flag = False
			\ENDIF
			\ENSURE flag 
		\end{algorithmic}
	\end{algorithm}

	\small\begin{proof} [Proof of the relation between $ \mne $ and $ \mnbe $] 
	
	Denote $ \mathcal{B}Q(s,a) = r(s,a) - \gamma\mathbb{E}_{s'}\max_{\tilde{a}}Q(s',\tilde{a}) $ as bellman operator. 
	\begin{align*}
	&\mne(Q) \\
	\le &\Vert Q(s,a) - \mathcal{B}Q^*(s,a)\Vert_\infty + \Vert \mathcal{B}Q^*(s,a) - Q^*(s,a)\Vert_\infty \\
	\le &\mnbe(Q) +  
 \Vert \gamma\mathbb{E}_{s'}\max_{\tilde{a}}Q(s',\tilde{a}) - \gamma\mathbb{E}_{s'}\max_{\tilde{a}}Q^*(s',\tilde{a})\Vert \\
	\le &\mnbe(Q) + \gamma \mne(Q)
	\end{align*} 
	So we can proof that
	$$ \mne(Q) \le \frac{\mnbe(Q)}{1-\gamma} .$$ 
\end{proof}\normalsize

	\section{Experiment}

		\begin{figure}[t]
			\captionsetup[subfigure]{labelformat=empty}
			\centering
			\subfloat[(a)]{
				\label{fig1}
				\includegraphics[width=1.5in,height=1.0in]{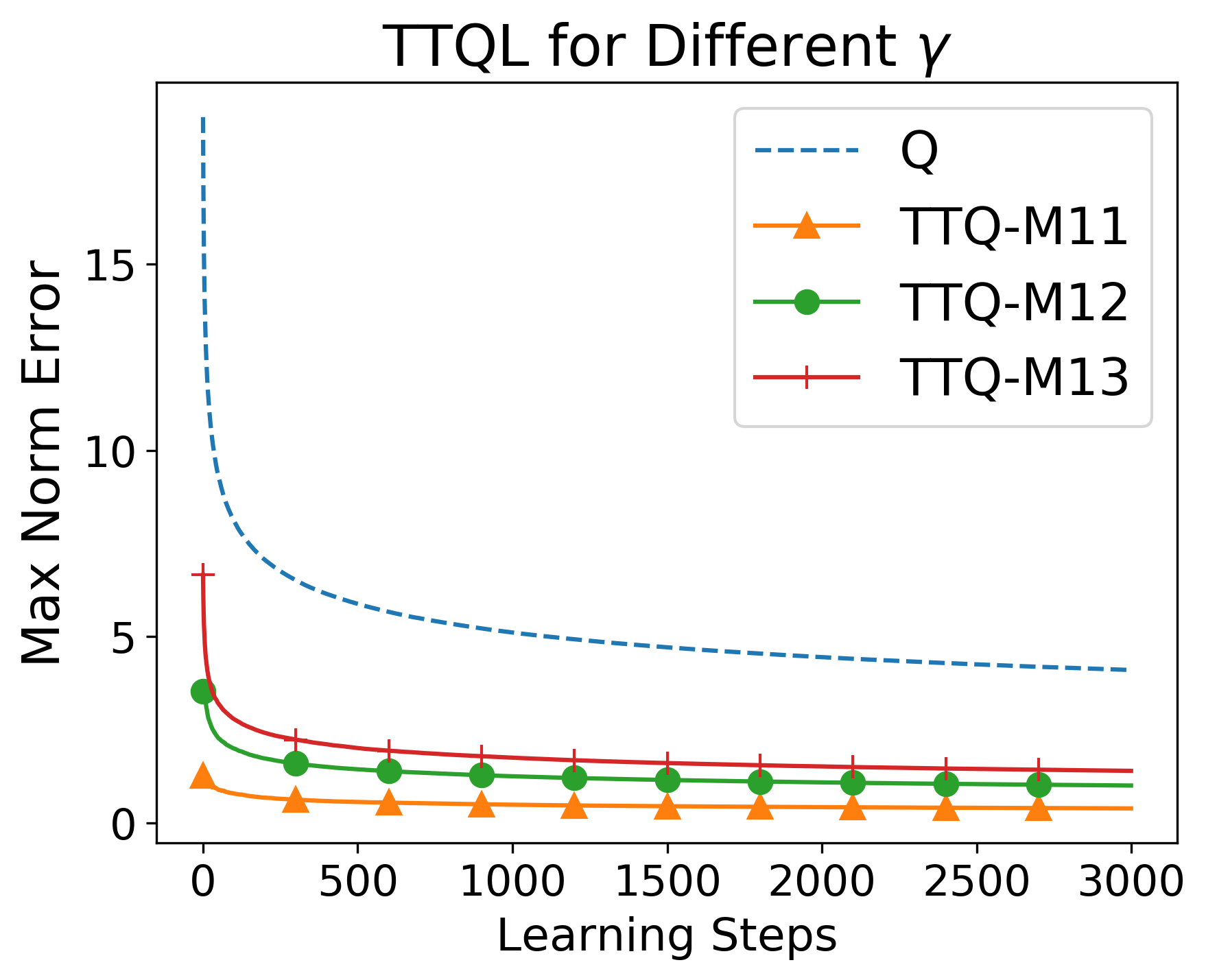}}  
			\subfloat[(d)]{
				\label{fig4}
				\includegraphics[width=1.5in,height=1.0in]{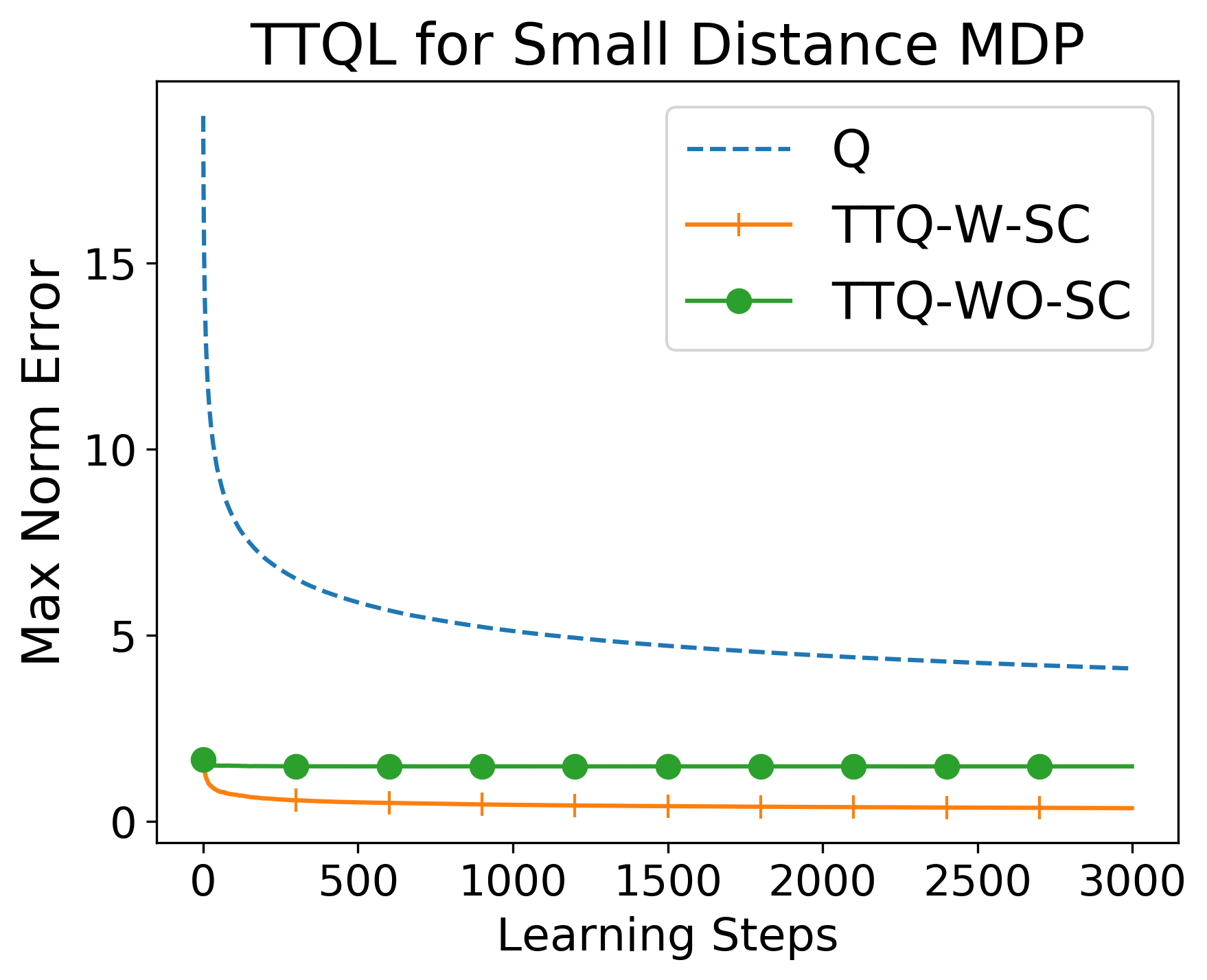}}
			
			\subfloat[(b)]{
				\label{fig2}
				\includegraphics[width=1.5in,height=1.0in]{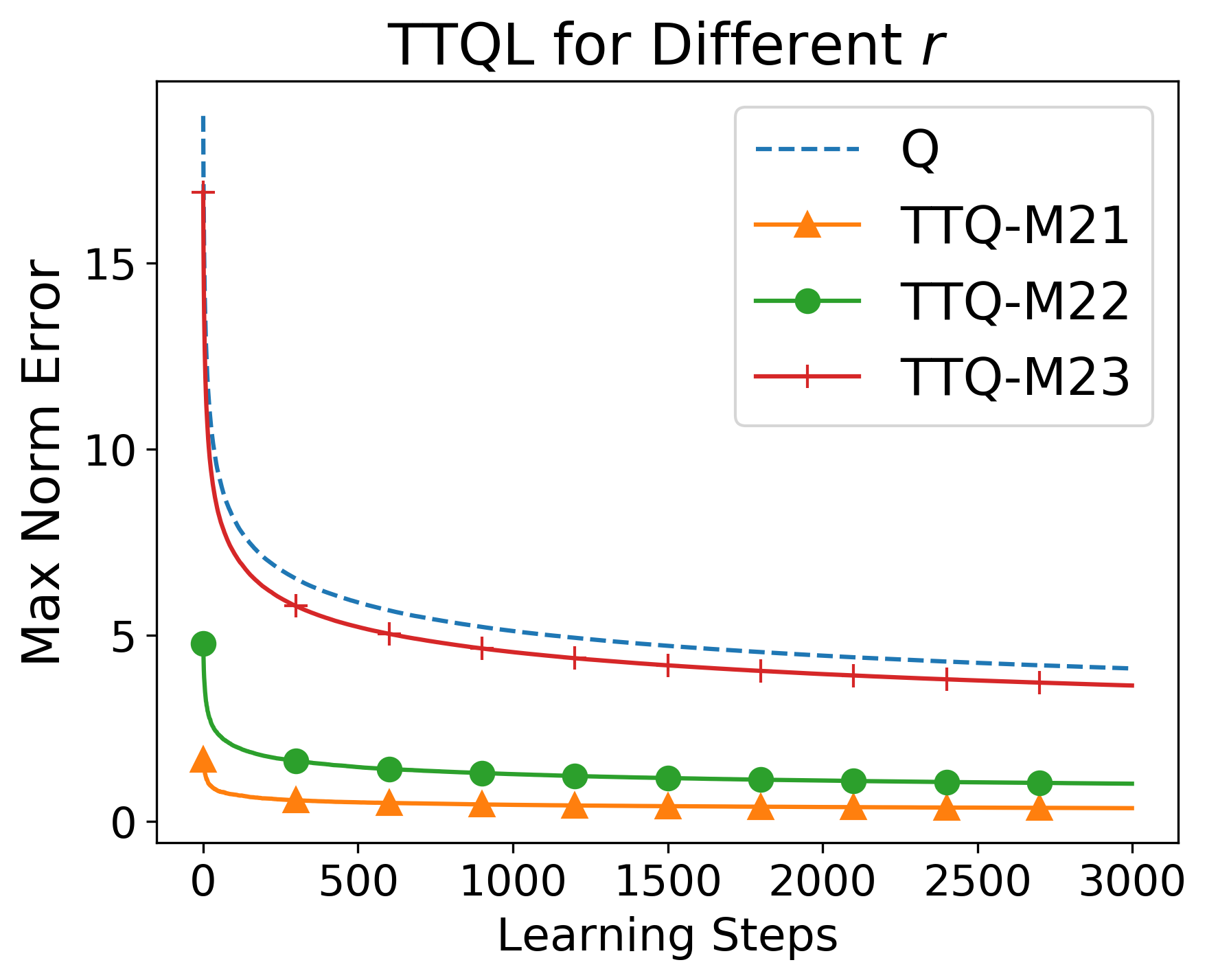}}
			\subfloat[(e)]{
				\label{fig5}
				\includegraphics[width=1.5in,height=1.0in]{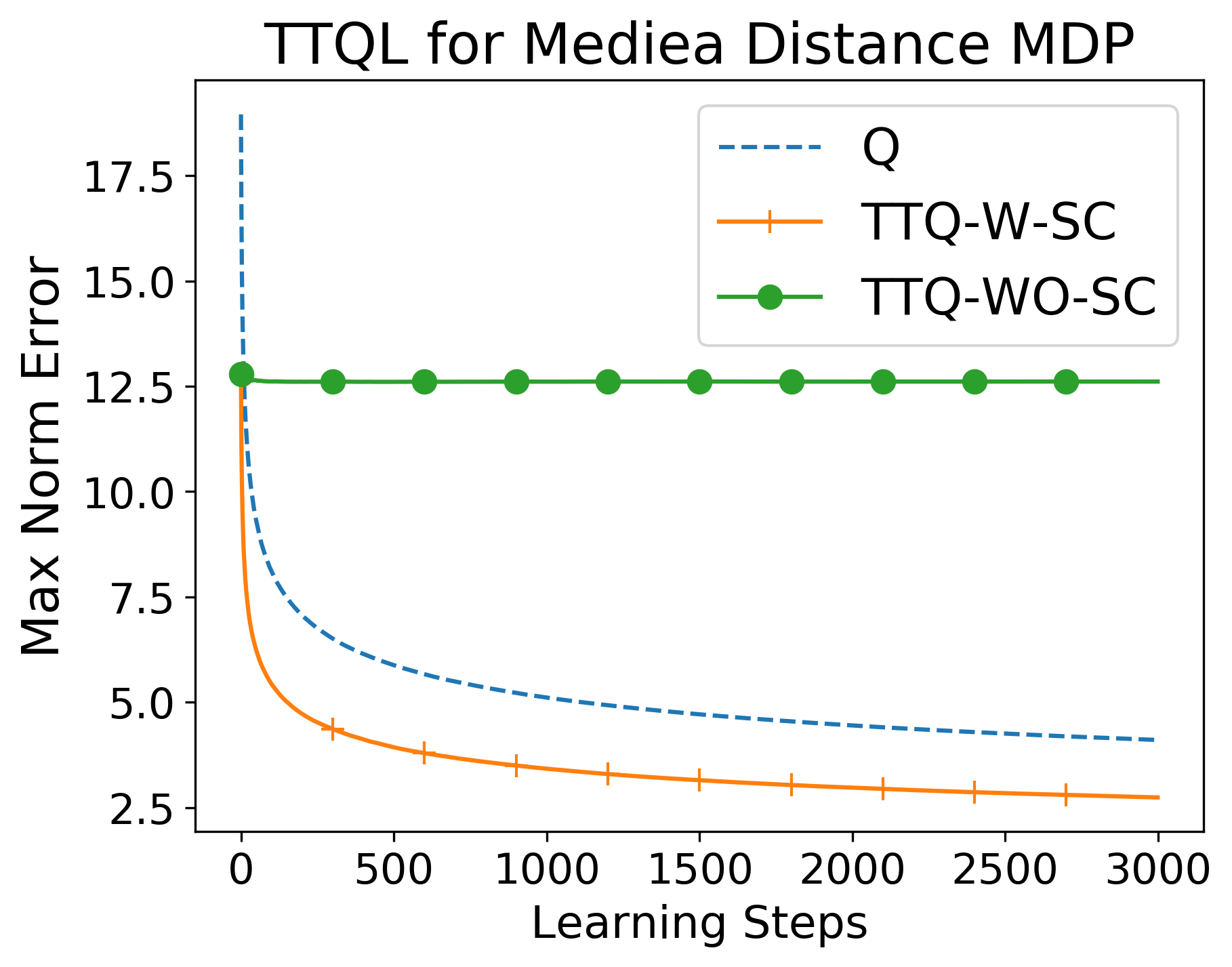}}
			
			\subfloat[(c)]{
				\label{fig3}
				\includegraphics[width=1.5in,height=1.0in]{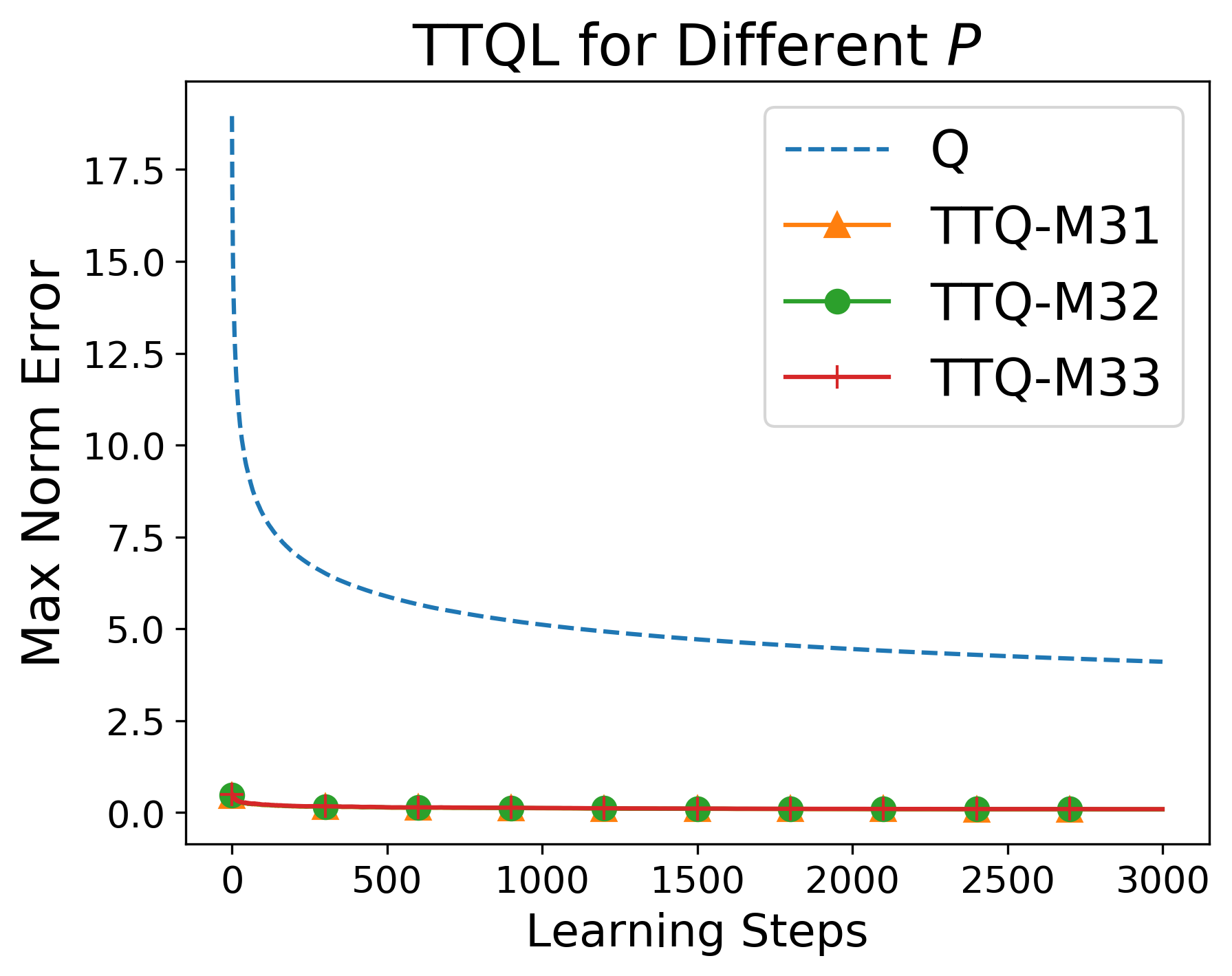}}
			\subfloat[(f)]{
				\label{fig6}
				\includegraphics[width=1.5in,height=1.0in]{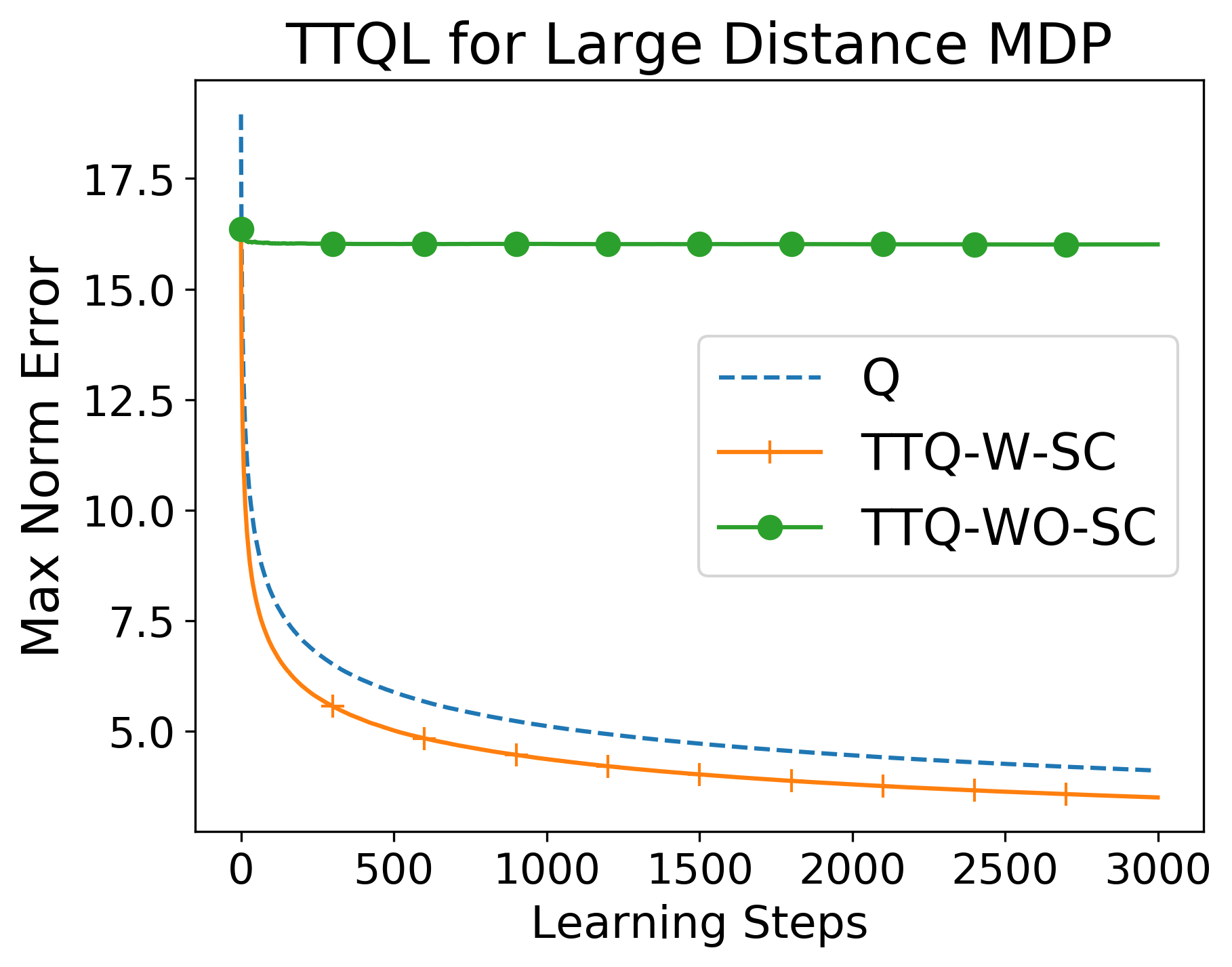}}
			\caption{  Left three figures are the learning errors w.r.t the three types of different MDPs ( Be different in $ \gamma $, $ r $, $ P $ respectively ).    Right three figures are the learning error w.r.t the three different distance transfer task  and both training with/without the safe condition. }
			\label{Fig1}
		\end{figure}

    In this section, we report our simulation experiments to support our convergence analysis and verified the effectiveness of our proposed target transfer Q-Learning with the error ratio safe condition. 

We consider the general MDP setting.  We construct the  random MDP by generating the transition probability $ P(s'|s,a) $, reward function $ r(s,a) $  and discount factor $ \gamma $ and fixing the state and action space size as 50.

First of all, we  generate 9 different  MDPs ($ M_{11}\sim M_{33} $) as source tasks and then generate the new MDP $ M_0 $.  Let $ M_{11},M_{12},M_{13}$ be different from $ M_0 $ in $ \gamma $ and the distance from $ M_{1\cdot} $ and $ M_0 $ increase as $ M_{11} < M_{12} < M_{13} .$       Similarly,  MDPs $ M_{21},M_{22},M_{23} $ is different from $ M_0 $ in $ r $, and   MDPs $ M_{31},M_{32},M_{33} $ is different from $ M_0 $ in $ P $. Then we run our algorithm to transfer the Q-function learned on these 9 source MPDs to the new MDP $ M_0 $. The result is shown in  Figure1a, 1b and  1c. Note that the  dash line $ Q $  is the  Q-learning algorithm with no transfer learning, and the  solid line with various markers are the TTQL algorithm.

Secondly, we design three   MDPs  $ M_{4}, M_{5}, M_{6}   $ as source task MDPs, and the distance between these MDPs and the target becomes larger and larger. Then we use TTQL to transfer the Q-function learning from them to new MDP $ M_0 $ with and without the safe condition.  The results is shown in  Figure1d, 1e and  1f. Note that $ W-SC $ means that the experiment is run with the safe condition and $  WO-SC $ means without the safe condition.

    We have the following observations. (1) TTQL method outperforms Q-learning in all experiments. (2) Running TTQL on the more similar MDPs will lead to the faster convergence rate. Note that the curve  in Figure \ref{fig3} are closed to each other. It is because the infinity norm of the $ P $ will be small because the scale of the $ P $ is small and is consistent with the Proposition \ref{propq*diff}.  (3) The safe condition is necessary to ensure the convergence of the algorithms in various situation. 
     All these observations are consistent with our theoretical findings.

	\section{Conclusion}
	
    In this paper, we proposed a new transfer learning in RL method  \emph{target transfer Q-learning}(TTQL). The method transfer the Q-function learned in the source task to the target of Q-learning in the new task   when the safe conditions are satisfied. We prove the TTQL method do converge with the safe condition and the convergence rate is quicker than Q-learning if the two MDPs are not faraway from each other. The theoretical analysis helps to design safe conditions which is key  to guarantee the convergence of TTQL. As far as we known, it is the first convergence rate guaranteed transfer leaning in reinforcement learning algorithm. In the future, we will apply the TTQL to the more complex tasks and study convergence rate for the TTQL with complex function approximation such as the neural network.

	\bibliography{accrl}
	\bibliographystyle{aaai}

\end{document}